\documentclass[10pt]{article} 
\usepackage[preprint]{tmlr}

\usepackage{amsmath,amsfonts,bm}

\def\eqref#1{equation~\ref{#1}}

\def\1{\bm{1}}

\DeclareMathAlphabet{\mathsfit}{\encodingdefault}{\sfdefault}{m}{sl}
\SetMathAlphabet{\mathsfit}{bold}{\encodingdefault}{\sfdefault}{bx}{n}

\usepackage{graphicx}
\usepackage{derivative}
\usepackage{tabto}
\usepackage{hyperref}
\usepackage{url}
\usepackage{graphicx}
\usepackage{amsmath}

\usepackage{amssymb}
\usepackage{booktabs}
\usepackage{algorithm}
\usepackage{algpseudocode}
\usepackage{ulem}
\usepackage{textcomp}
\usepackage{wrapfig}
\usepackage{amsmath}
\usepackage{bm}
\usepackage{dsfont}
\usepackage{xcolor}
\usepackage{soul}

\newtheorem{proof}{Proof}
\newtheorem{theorem}{Theorem}

\newtheorem{lemma}{Lemma}
\usepackage[utf8]{inputenc}

\title{Improving Adversarial Training using Vulnerability-Aware Perturbation Budget}

\author{\name Olukorede  Fakorede\email fakorede@iastate.edu \\
      \addr Department of Computer Science\\
      Iowa State University
      \AND
      \name Modeste Atsague\email modeste@iastate.edu \\
      \addr Iowa State University
      \AND
      \name Jin Tian \email jtian@iastate.edu\\
      \addr Iowa State University}

\def\month{MM}  
\def\year{YYYY} 
\def\openreview{\url{https://openreview.net/forum?id=XXXX}} 

\begin{document}

\maketitle

\begin{abstract}
 Adversarial Training (AT) effectively improves the robustness of Deep Neural Networks (DNNs) to adversarial attacks. Generally, AT involves training DNN models with adversarial examples obtained within a pre-defined, fixed perturbation bound. Notably, individual natural examples from which these adversarial examples are crafted exhibit varying degrees of intrinsic vulnerabilities, and as such, crafting adversarial examples with fixed perturbation radius for all instances may not sufficiently unleash the potency of AT. Motivated by this observation, we propose two simple, computationally cheap vulnerability-aware reweighting functions for assigning perturbation bounds to adversarial examples used for AT, named  \textit{Margin-Weighted Perturbation Budget (MWPB)} and  \textit{Standard-Deviation-Weighted Perturbation Budget (SDWPB)}. The proposed methods assign perturbation radii to individual adversarial samples based on the vulnerability of their corresponding  natural examples. Experimental results show that the proposed methods yield genuine improvements in the robustness of  AT algorithms against various adversarial attacks.
\end{abstract}

\section{Introduction}

In recent years, Deep Neural Networks (DNNs) have demonstrated remarkable success across various domains, achieving impressive performance benchmarks. However, this success has come hand in hand with a critical concern: the vulnerability of DNNs to well-crafted adversarial perturbations \cite{szegedy2013intriguing,goodfellow2014explaining}. This observed brittleness has raised questions about the safe deployment of DNNs in safety-critical applications.

In response to the challenge of adversarial vulnerability, a multitude of defense mechanisms have been proposed to enhance the robustness of DNNs. Among these, adversarial training (AT) \cite{goodfellow2014explaining,madry2017towards} stands out as one of the most prominent and effective approaches. AT typically involves training neural networks using adversarial examples. The effectiveness of AT has inspired many variants such as \cite{zhang2019theoretically,wang2019improving,ding2019mma,zhang2020geometry,zeng2021adversarial}, among others. 
Moreover, alternative methods, such as adversarial weight perturbation \cite{wu2020adversarial}, instance re-weighting  \cite{zhang2020geometry,liu2021probabilistic,fakorede2023vulnerability}, and hypersphere embedding  \cite{pang2020boosting,fakorede2023improving}, have emerged to further enhance the performance of existing AT variants.

It has been established that the efficacy of adversarial training varies significantly across samples of various classes \cite{xu2021robust,fakorede2023vulnerability,wei2023cfa}. Adversarial examples derived from natural samples that are inherently vulnerable are significantly misclassified in adversarially trained models \cite{liu2021probabilistic,zhang2020geometry}. 
Specifically, the robust accuracy achieved when evaluating adversarial samples originating from natural examples closer to class boundaries is substantially lower than adversarial samples stemming from inherently more robust natural examples. 
In response to these challenges, various reweighting techniques have been introduced to enhance the effectiveness of AT \cite{zhang2020geometry,liu2021probabilistic,fakorede2023vulnerability}.  These techniques operate by assigning largerr weights to  the losses of disadvantaged examples.  

It's important to highlight that current AT methods predominantly rely on adversarial examples generated with predetermined, fixed perturbation radii. However, the notable performance variations observed across different adversarial examples prompt a fundamental question: \textit{ Is uniform perturbation of adversarial examples used in AT necessary or beneficial, as is commonly practiced in existing research?} For instance, \textit{should adversarial examples crafted from inherently vulnerable natural samples be allocated the same perturbation budget as those derived from  more robust examples?}

In this work, we present a case against applying uniform perturbation. We demonstrate that first-order adversarial attacks, such as projected gradient descent (PGD), induce a higher increase in adversarial loss for adversarial samples originating from vulnerable natural examples compared to those derived from inherently robust natural examples when subjected to uniform perturbation radii. We also argue that enlarging the perturbation radii may increase the inner maximization loss of adversarial examples derived from inherently robust natural examples. Consequently, we introduce  reweighting methods designed to allocate varying perturbation budgets to individual adversarial examples employed in adversarial training. Our proposed methods assign these budgets based on the vulnerabilities exhibited by their corresponding natural examples. We employ two methods for estimating each natural sample's vulnerability: one  leveraging logit margin, and the other utilizing a modified standard deviation measure of the output logits. 

Our rationale is as follows: Natural examples that are intrinsically vulnerable may require relatively smaller perturbations to yield effective adversarial examples suitable for training. Conversely, when crafting adversarial examples from intrinsically robust natural examples, relatively larger perturbations may be necessary to achieve better training impact. Experimental results show that the proposed methods improve the performance of existing AT methods including standard AT \cite{madry2017towards}, TRADES \cite{zhang2019theoretically}, and MART \cite{wang2019improving}.

We summarize the contributions of our work as follows:

\begin{enumerate}

\item  We argue for assigning 
varying perturbation radii to individual adversarial samples based on the vulnerability of their corresponding natural examples 
in the inner maximization component of the min-max adversarial training framework. 

\item Consequently, we propose two  reweighting functions for assigning perturbation radii to individual adversarial examples based on their vulnerability.

 \item We empirically demonstrate the effectiveness of the proposed strategy in improving adversarial training and show its superiority over existing reweighting and adaptive perturbation radii methods, especially against strong white-box and black-box attacks. 
 
\end{enumerate}

\section{Related Work}

\subsection{Adversarial Robustness.}
Adversarial robustness is a model's ability to withstand adversarial attacks. Many methods \cite{guo2018countering,papernot2017practical,madry2017towards,goodfellow2014explaining,zhang2019theoretically} have been proposed to improve adversarial robustness of neural networks. However, some of these methods are ineffective against stronger attacks. Adversarial training (AT) \cite{madry2017towards}, which involves training the model with adversarial examples obtained under worst-case loss, has significantly improved robustness. Formally, AT involves solving a min-max optimization as follows:
\begin{align} \label{min-max}
        \min_{\bm{\theta}} \mathbb{E}_{(\textbf{x},y) \sim \mathcal{D}}   \left[ \max_{\textbf{x}' \in B_{\epsilon}(\textbf{x})} L(f_{\bm{\theta}}(\textbf{x}'), y) \right]
\end{align} 
where $y$ is the true label of input feature $\textbf{x}$, $L()$ represents the loss function, $\bm{\theta}$ are the model parameters, and \(B_{\epsilon}(\textbf{x}) : \{\textbf{x}' \in \mathcal{X}: \|{\textbf{x}}' - \textbf{x}\|_p \leq \epsilon \}\) represents the $l_p$ norm ball centered around $\bf{x}$ constrained by radius $\epsilon$ . In Eq. (\ref{min-max}), the inner maximization tries to obtain a worst-case adversarial
version of the input $\bf{x}$ that increases the loss. The outer minimization then tries to find
model parameters that would minimize this worst-case adversarial loss. The relative success of AT has inspired an array of variants including prominent TRADES \cite{zhang2019theoretically} and MART \cite{wang2019improving}, and a few others \cite{wu2020adversarial,pang2020boosting}.


\subsection{Reweighting}
Recent works have advocated for assigning unequal weights to the inner maximization loss \cite{zeng2021adversarial} and robust losses \cite{liu2021probabilistic,zhang2020geometry} to improve the performance of AT. It has been shown that adversarially-trained models poorly classify adversarial examples crafted from natural examples that are intrinsically harder to classify \cite{xu2021robust,fakorede2023vulnerability}. Most existing reweighting methods attempt to improve AT by upweighting the robust losses corresponding to vulnerable adversarial examples. While some of these methods improve model robustness to certain attacks, they have performed quite poorly in defending against strong attacks  \cite{fakorede2023vulnerability}. 

In contrast with existing reweighting approaches focusing on reweighting losses, this paper 
proposes a novel approach of reweighting the perturbation radii of adversarial examples used for AT. 
Unlike previous works that attempt to improve robust accuracy by assigning larger weights to examples closer to the decision boundary, this work 
improves AT by assigning larger perturbation budgets to adversarial examples crafted from inherently robust natural examples.  

\subsection{Adaptive Perturbation Radii.}
Few works in the literature have been directed at adaptive perturbation radii for adversarial training. Notably, Ding et al.cite{ding2019mma} proposed MMA that maximizes margin to achieve robustness while also adaptively selecting the ``correct`` perturbation radius for each data point. The ``correct`` radius for each data point is characterized by the ``shortest successful perturbation`` to misclassify the data point. Similarly, Balaji et al. cite{balaji2019instance} proposed IAAT that begins with using a perturbation size ${\epsilon}_i$ that is as large as possible, then adaptively adjusts  ${\epsilon}_i$ depending on whether PGD succeeds in finding a misclassified label at ${\epsilon}_i$.

In contrast to these works that perform an exhaustive search for suitable perturbation radii, our proposed methods use pre-defined reweighting functions for adaptively assigning perturbation radii to individual examples at no extra cost. In addition, MMA  and IAAT  are designed to keep perturbed samples on the decision boundary without pushing farther into incorrect classes, whereas our methods assign significantly larger perturbations to intrinsically robust examples. 
Most importantly,  MMA  failed to improve adversarial robustness over standard AT against strong adversarial attacks like \textit{Auto Attacks}.\footnote{The comparison results with MMA is shown in Table~\ref{table:margin-other}. The self-reported IAAT robustness for PGD attack is lower compared to standard AT \cite{madry2017towards}.} 

\section{Preliminaries}

We use bold letters to represent vectors. 
 We denote \(\mathcal{D} = \{\textbf{x}_i ,y_i\}_{i=1}^n\)  a data set of input feature vectors \(\textbf{x}_i \in \mathcal{X} \subseteq \mathbf{R}^d\) and labels \(y_i \in \mathcal{Y}\), where \(\mathcal{X}\) and \(\mathcal{Y}\) represent a feature space and a label set, respectively. 
 
 Let \(f_{\theta}:\mathcal{X} \to R^{C} \) denote a deep neural network (DNN) classifier with parameters  $\theta$, and  $C$ represents the number of output classes. For any \(\mathbf{x} \in \mathcal{X}\), let the class label predicted by $f_{\theta}$ be   $F_{\theta}(\mathbf{x}) =  \arg \max_{k} f_{\theta}(\mathbf{x})_k$, where $f_{\theta}(\mathbf{x})_k$ denotes the $k$-th component of $f_{\theta}(\mathbf{x})$. $f_{\theta}(x)_y$ is the probability of $x$ having label $y$. 

We denote \(\|\cdot\|_p\) as the \(l_p\)- norm over \(\mathbf{R}^d\), that is, for a vector \(\mathbf{x} \in \mathbf{R}^d,  \|\textbf{x}\|_p = (\sum^d_{i=1}|\textbf{x}_i|^p)^{\frac{1}{p}}\). An \(\epsilon\)-neighborhood for \textbf{x} is defined as \(B_{\epsilon}(\mathbf{x}) : \{\mathbf{x}' \in \mathcal{X}: \|{\mathbf{x}}' - \mathbf{x}\|_p \leq \epsilon \}\). An adversarial example corresponding to a natural input $\mathbf{x}$ is denoted as $\mathbf{x}'$. We often refer to the loss resulting from the adversarial attack (inner maximization) as adversarial loss.

\section{Proposed Approach}
In this section, we motivate and present two novel perturbation budget assignment functions for the inner maximization of the AT framework.

\subsection{A Case Against Fixed Perturbation.}
AT employs a min-max optimization procedure aimed at minimizing worst-case losses computed over perturbed adversarial examples. While the application of AT  significantly enhances model robustness, it is essential to acknowledge that the optimized worst-case losses represent approximations that may sometimes be suboptimal \cite{mao2023taps}. This limitation arises from the nature of the Projected Gradient Descent (PGD) algorithm, which is prone to converging to  local maxima scattered across the optimization landscape.


Furthermore, the geometric characteristics inherent to individual natural data samples utilized in the inner maximization step of the min-max optimization  exhibit substantial variation. Some natural samples inherently fall into misclassification regions or exhibit elevated loss values. For these particular examples, the PGD algorithm may readily identify suitable maxima. Conversely, naturally robust samples may not yield the worst-case loss under similar optimization settings as their more vulnerable counterparts. 

The pursuit of adversarial examples that optimally maximize losses represents a compelling objective for achieving highly robust models. Nevertheless, it is important to recognize that the prevailing practice involves applying uniform optimization settings for the inner maximization step, irrespective of the individual idiosyncrasies inherent to each original sample. Specifically, it is commonly assumed that the maximal loss for each example can be discovered using the same perturbation radius. However, it becomes evident that, under this uniform perturbation radius, the losses observed for adversarial examples derived from challenging natural samples exhibit an increased discrepancy compared to those computed for natural samples situated farther from the decision boundary. This assertion warrants a theoretical examination, which we delve into below. 

\begin{theorem} 
Let $\mathcal{L}$ and $f_{\theta}(.)$ denote the cross-entropy loss function and the predictions of the model respectively. Consider two natural input-label pairs $(x_1, y_1)$ and $(x_2, y_2)$ 
such that $\mathcal{L}(x_1, y_1) > \mathcal{L}(x_2, y_2)$. 
The following holds for  first-order adversarial examples  crafted from $x_1$ and $x_2$ within the same perturbation radius $\epsilon$:
\begin{enumerate}
\item $\mathcal{L}(f_{\theta}({x}'_1), y_1) > \mathcal{L}(f_{\theta}({x}'_2), y_2) $
\item $\mathcal{L}(f_{\theta}({x}'_1), y_1) - \mathcal{L}(f_{\theta}(x_1), y_1) > \mathcal{L}(f_{\theta}({x}'_2), y_2) - \mathcal{L}(f_{\theta}(x_2), y_2)$
\end{enumerate}

\end{theorem}

\textit{\textbf{Remark}.} \textit{Theorem 1 underscores that when subjected to the same perturbation radius, adversarial examples stemming from vulnerable (with high loss values) natural examples   incur a relatively more substantial loss increase. Similarly, adversarial examples generated from robust natural examples with lower natural losses induce relatively smaller loss increments than adversarial examples from vulnerable natural examples. Furthermore, it can be inferred that under uniform perturbation, adversarial examples from various natural examples have varying loss increments.} 


Finding adversarial examples with better maxima is associated with better adversarial robustness \cite{madry2017towards}. Therefore, using a uniform perturbation radius for the inner-maximization may not yield the best robustness. The reason is due to the considerable variance in the inner maximization losses of individual examples under uniform perturbation radius. 
Additionally, considering that adversarial examples originating from inherently robust (low-loss) natural examples tend to result in relatively smaller increases in loss, we suggest an approach where we assign varying perturbation radii to each instance of training example based on their natural vulnerability. We demonstrate in \textit{Theorem 2} that enlarging the perturbation radii around an example can increase the loss.

\begin{theorem}
    The inner maximization  $\max_{\textbf{x}' \in B_{\epsilon}(\textbf{x})} L(f_{\theta}(\textbf{x}'), y)$ increases as $\epsilon$ increase.
\end{theorem}

\subsection{Vulnerability-Aware Reweighting for Perturbation Radii}

In the preceding section 4.1, we presented a rationale against the use of uniform perturbation radii for crafting adversarial examples employed in AT. Instead, we advocate for assigning distinct perturbation radii for generating adversarial examples 
based on the inherent vulnerabilities of their original natural examples. 
Here, we propose two measures of estimating the intrinsic vulnerability of individual natural examples from two perspectives: (1) the geometric proximity of each example to the decision boundary, captured using logit margins, and (2) the standard deviation of the DNN's output logits.

\subsubsection{Margin-based Vulnerability Estimation}

Measuring the exact proximity of a data point to the decision boundary is not straightforward for non-linear models like DNNs. We adopt a measure of \textit{multi-class margin} described in \cite{koltchinskii2002empirical} to estimate the vulnerability or robustness of natural examples. Consider the predictions of a DNN denoted by $f_{\theta}$ and a labelled example $(x, y)$, the margin $d_{m}(x, y; \theta)$ is given as follows: 

  \begin{equation}\label{margin}
   d_{m}(\textbf{x}, y; \theta)= f_{\theta}(\textbf{x})_y  - \max_{ k, k \neq y} f_{\theta}(\textbf{x})_k
   \end{equation}
where   $f_{\theta}(\textbf{x})_y$ is the model's predicted probability of the correct label $y$, and $\max_{ k, k \neq y} f_{\theta}(\textbf{x})_k$ is the largest prediction of the remaining classes.

We utilize the information provided by $d_{m}(\textbf{x}, y; \theta)$ in measuring the vulnerability of a natural input example $\textbf{x}$ as follows:


\begin{itemize}
\item If $d_{m}(\textbf{x}, y; \theta) > 0$, $\textbf{x}$ is correctly classified and scored. We consider $\textbf{x}$ relatively robust.  
\item If $d_{m}(\textbf{x}, y; \theta) = 0$, it implies that $\textbf{x}$ has the same prediction score as the best of the remaining classes. As such, we consider $\textbf{x}$ to be located at the class boundary.
\item If $d_{m}(\textbf{x}, y; \theta) < 0$, $\textbf{x}$ is considered to be vulnerable, since it is located in a wrong region even before $\textbf{x}$ is adversarially perturbed.  
\item In addition to identifying whether a sample $\textbf{x}$ is vulnerable, we assess the degree of vulnerability based on the magnitude of the value returned by $d_{m}(\textbf{x}, y; \theta)$. For example, if $d_{m}(\textbf{x}1, y_1; \theta) > d_{m}(\textbf{x}_2, y_2; \theta)$, we infer that $\textbf{x}_2$ is relatively more vulnerable than $\textbf{x}_1$.
\end{itemize}



\subsubsection{Standard-Deviation-based Vulnerability Estimation}
 In this section, we propose estimating the vulnerability of individual examples using a modified standard deviation of a DNN's model output logits.

The standard deviation serves as a metric to assess the distribution spread, where a smaller standard deviation implies a more uniform distribution. In the context of a model's output logits on a given input, an evenly spread distribution indicates a higher risk of misclassification. This is because the model's estimated probabilities for an input, with a more even spread among both correct and incorrect classes, suggest potential vulnerability to misclassification. It's noteworthy that previous research has established connections between various variance estimations and the difficulty or susceptibility of input examples \cite{agarwal2022estimating} as well as classes \cite{xu2021robust}. In this context, we leverage the standard deviation of predicted probabilities to gauge the vulnerability of individual samples.

Conventionally, the standard deviation is measured as the variation of random variables around the mean of the distribution. Here, we modify the original standard deviation formula by replacing the mean with the model's predicted probability of the example belonging to the  true class as follows: 
\begin{align} \label{standard-dev}
        d_{std}(\textbf{x}_i, y_i, \theta) = \sqrt{\frac{\sum_{k=1}^{C}  (f_{\theta}(\textbf{x}_i)_k - f_{\theta}(\textbf{x}_i)_{y_i})^2}{|C|}}
\end{align}
where $x_i$ and $y_i$ represent the input of a sample and its corresponding label, and $C$ is the number of classes. 
The proposed formulation in Eq. (\ref{standard-dev}) measures the spread of the model's predicted probabilities around the model's predicted probability of the true class. A low $d_{std}(x_i, y_i, \theta)$ implies that 
the model is not confident in its estimated probability of $x_i$ belonging to $y_i$, indicating higher risk and vulnerability of misclassification.

Contrary to the logit margin approach in Eq.(\ref{margin}), which calculates the disparity between the probability of the correct class for a sample and the probability of the nearest incorrect class, the suggested modified standard deviation considers the model's predictions for all classes.

\subsection{Weight Assignment for Perturbation Radii}
In Theorem 1, we establish that the inner maximization process inherently induces a more substantial increase in loss for samples that possess intrinsically high loss values. This holds true for misclassified natural samples, characterized by negative margins and exhibiting high loss values. Furthermore, considering that these misclassified examples are already situated within regions of high loss, we advocate for a strategy in which the inner maximization process for misclassified examples employs smaller perturbation radii. 
Conversely, the inner maximization process inherently leads to a relatively lower increase in loss for samples with low loss values, typically associated with positive margins. Consequently, we propose the utilization of larger perturbation radii for generating adversarial examples from these low-loss samples.

Therefore, we introduce two vulnerability-aware radius reweighting functions based on the two measures in Eq.~(\ref{margin}) and (\ref{standard-dev}) respectively. The first reweighting function, termed \textbf{Margin-weighted Perturbation Budget (MWPB)}, is formulated as follows:
  \begin{equation}\label{margin-weighted-radii}
   {\epsilon}_{i} = exp(\alpha \cdot d_{m}(\textbf{x}_i, y_i; \theta)) * \epsilon
   \end{equation}
where $\epsilon = 8/255$, the commonly used fixed perturbation radius parameter,  $d_{m}(\textbf{x}_i, y_i; \theta)$ is given by Eq.~(\ref{margin}), 
and $\alpha$ is a hyperparameter that controls the weight of the function.

The  \textit{MWPB} reweighting  function in Eq. (\ref{margin-weighted-radii}) allocates larger perturbation radii $({\epsilon}_i > 8/255)$ when generating adversarial examples from natural samples characterized by positive margins. In contrast, it assigns smaller perturbation radii $({\epsilon}_i < 8/255)$ when crafting adversarial examples from samples with negative margins. 
Finally, adversarial examples originating from samples situated exactly at the class boundary are crafted using the default perturbation radii $({\epsilon}_i = 8/255)$.

The second reweighting function, termed \textbf{Standard-Deviation-Weighted Perturbation Budget (SDWPB)}, is as follows:
  \begin{equation}\label{std-weighted-radii}
   {\epsilon}_{i} = exp(\alpha \cdot d_{std}(\textbf{x}_i, y_i; \theta)) * \epsilon
   \end{equation}
where $d_{std}(\textbf{x}_i, y_i; \theta)$ is given by Eq.~(\ref{standard-dev}).    
The above \textit{SDWPB} reweighting function allocates comparatively larger perturbation radii when generating adversarial examples from natural instances with relatively larger $d_{std}(\textbf{x}_i, y_i; \theta)$ values. Unlike the \textit{MWPB} reweighting, \textit{SDWPB} does not assign perturbation radii ${\epsilon}_i$ to values less than $8/255$ because the proposed $d_{std}(\textbf{x}_i, y_i; \theta)$ is non-negative. However, both \textit{MWPB} and \textit{SDWPB} assign relatively larger perturbation radii to samples that are deemed to be  naturally more robust based on the metrics defined in Eq. (\ref{margin}) and (\ref{standard-dev}) respectively.

As seen in Theorem 1, first order adversarial examples crafted from relatively robust natural examples incur relatively smaller increment in losses. By assigning relatively larger perturbation radii for crafting adversarial examples from more robust natural examples, the reweighting functions in Eq. (\ref{margin-weighted-radii}) and (\ref{std-weighted-radii}) ensure that larger inner maximization losses for these adversarial examples. We show in Theorem 2 that increasing perturbation radii can increase the inner maximization loss.


\subsection{Applying the  Proposed Weighted Perturbation Budget Methods}
Every adversarial training method is a variant of min-max optimization. Hence, our proposed reweighting methods may be applied to any adversarial training variant. We re-write the min-max adversarial training objective in Eq. (\ref{min-max}) as follows:
\begin{align} \label{mwpr-min-max}
        \min_{\bm{\theta}} \frac{1}{n} \sum_{i=1}^{n}  \left[ \max_{\textbf{x}_i' \in B_{{\epsilon}_i}(\textbf{x}_i)} L(f_{\theta}(\textbf{x}_i'), y_i) \right]
\end{align}
where each ${\epsilon}_i$ is computed according to Eq. (\ref{margin-weighted-radii}) (MWPB) or Eq. (\ref{std-weighted-radii}) (SDWPB) for each input-label pair $(\textbf{x}_i, y_i)$, and \(B_{{\epsilon}_i}(\mathbf{x_i}) : \{\textbf{x}_i' \in \mathcal{X}: \|{\textbf{x}_i}' - \textbf{x}_i\|_p \leq {\epsilon}_i \}\). For the purpose of our experiments, which we present in Section 5, we apply MWPB and SDWPB to popular existing AT variants standard \textit{AT} \cite{madry2017towards}, \textit{TRADES} \cite{zhang2019theoretically}, and \textit{MART} \cite{wang2019improving}. 

\subsection{Challenge of Adversarial Training with Larger Perturbation Radii}
The adversarial loss landscape is unfavorable to optimization under large perturbation budgets  \cite{liu2020loss}. It is shown that when a perturbation size $\epsilon$ is large, the gradients become small due to decreased gradient magnitude in the initial sub-optimal region, making it challenging for the model to escape the sub-optimal initial region.
Large perturbation size may also encourage a model to find sharper minima. 
In contrast, smaller perturbation budgets facilitate larger gradient magnitude in the initial sub-optimal regions, which in turn help prevent the model from getting stuck in these regions.

Given that our proposed method requires perturbing a subset of the training data with relatively larger perturbation budgets, we use a two-phase training approach. We train the initial epochs with adversarial examples obtained under a smaller perturbation budget of ${\epsilon} / 2$. This allows the model to gradually adapt. Subsequently, we transition to the adversarial training objective  in Eq. (\ref{mwpr-min-max}), which employs larger perturbation budgets. We provide the algorithms for \textit{MWPB-AT} and \textit{SDWPB-AT} in the \textit{Appendix C}. 

\section{Experiments}
In this section, we extensively evaluate the proposed method. To test the versatility of our method, we test on various datasets including CIFAR-10 \cite{krizhevsky2009learning}, SVHN \cite{netzer2011reading}, and TinyImageNet \cite{deng2009imagenet}. We apply simple data augmentations such as 4-pixel padding with 32 × 32 random crop and random horizontal flip on CIFAR-10 and Tiny Imagenet. We utilize Resnet-18 \cite{he2016deep} and Wideresnet-34-10 \cite{he2016deep} as the backbone models. For brevity, we respectively denote ResNet-18 and Wideresnet-34-10 as RN18 and WRN34-10.

\subsection{Experimental Setup}
\subsubsection{Training Parameters.} We trained the  networks using mini-batch gradient descent for 110 epochs, with momentum 0.9 and batch size 128. We use the weight decay of 5e-4 for training CIFAR-10 and 3.5e-3 for SVHN and Tiny Imagenet. The  initial learning rate is set to 0.1 (0.01 for SVHN and Tiny Imagenet), and divided by 10 in the 80-th epoch, and then at the 90-th epoch. We train the first 80 epochs with adversarial examples obtained via PGD with a smaller perturbation budget of $4/255$ and step size of $1/ 255$. Subsequently, we introduce MWPB-AT / SDWPB-AT, MWPB-TRADES / SDWPB-TRADES, and MWPB-MART / SDWPB-MART in the 81-st epoch to improve AT 
 \cite{madry2017towards}, TRADES \cite{zhang2019theoretically} and MART \cite{wang2019improving} respectively. 

\subsubsection{Hyperparameters.} 
The values of $\alpha$ in Eq. (\ref{margin-weighted-radii}) and (\ref{std-weighted-radii}) are determined heuristically for each of AT methods. 

\textbf{MWPB.} In MWPB-AT, MWPB-TRADES, and MWPB-MART on CIFAR-10, the $\alpha$ values are 0.58, 0.42, and 0.55, respectively. For Tiny Imagenet, the corresponding $\alpha$ values are 0.55, 0.4, and 0.7. On SVHN, 
the corresponding $\alpha$ values are 0.5, 0.15, and 0.6.

\textbf{SDWPB.} For SDWPB-AT, SDWPB-TRADES, and SDWPB-MART on CIFAR-10, the $\alpha$ values are 0.62, 0.5, and 0.52. On Tiny Imagenet, the values are 0.5, 0.5, and 0.8. On SVHN, 
the $\alpha$ values are 0.6, 0.2, and 0.7.

We  show how the hyperparameter $\alpha$ was selected and study the influence of  $\alpha$  on the natural and robust accuracy for MWPB-AT/SDWPB-AT on CIFAR-10 using WRN-34-10. We provide the results and more explanations in Appendix A.3. 

\subsubsection{Baselines.} Our baselines include Standard AT  \cite{madry2017towards}, TRADES 
 \cite{zhang2019theoretically}, and MART 
  \cite{wang2019improving}. Furthermore, we conduct a comparative analysis of our approach against  MMA \cite{ding2019mma}, which also introduces adaptive perturbation radii to enhance adversarial robustness.   Lastly, we compare our results to other works that utilize logit-margin for improving adversarial robustness \textit{MAIL} \cite{liu2021probabilistic}, \textit{WAT} \cite{zeng2021adversarial}, AWP \cite{wu2020adversarial} and ST-AT \cite{li2023squeeze}. All the hyperparameters of the baselines are the same as  in their original papers. Nevertheless, we maintain consistency by using the same learning rate, batch size, and weight decay values as those utilized during the training of our proposed method.

\subsection{Threat Models}
We assess the performance of the proposed method attacks under \textit{White-box} and \textit{Black-box} settings and \textit{Auto attack}.

\textbf{White-box attacks.} These attacks have access to model parameters. To evaluate robustness on CIFAR-10 using RN-18 and WRN34-10, we apply the PGD attack with $\epsilon = 8/255$, step size $\kappa$ = $1/255$, $K = 20$; CW (CW loss \cite{carlini2017towards} optimized by PGD-20) attack with $\epsilon = 8/255$, step size $1/255$. 
On SVHN and Tiny Imagenet, we apply PGD attack with $\epsilon = 8/255$, step size $\kappa$ = $1/255$, $K = 100$. 

\textbf{Black-box attacks.} An adversary does not have  access to the model parameters under  black-box settings. We tested the robust models trained on CIFAR-10 against strong black-box attacks  Square \cite{andriushchenko2020square} with 5,000 queries and SPSA \cite{uesato2018adversarial} with 100 iterations,
perturbation size of 0.001 (for gradient estimation), learning rate = 0.01, and 256 samples for each gradient
estimation. All black-box evaluations are made on trained WRN34-10.

\textbf{Auto attacks.} Lastly, we evaluated the  trained models on \textit{Autoattack} ($l_{\infty}$ and $l_2$)\cite{croce2020reliable}, which   is a powerful ensemble of attacks consisting of APGD-CE \cite{croce2020reliable}, APGD-T \cite{croce2020reliable}, FAB-T \cite{croce2020minimally}, and Square (a black-box attack) \cite{andriushchenko2020square} attacks.

\subsection{Performance Evaluation}
We present CIFAR-10 results using RN18 and WRN34-10 in Tables \ref{table:cifar10-rn18} and \ref{table:wrn-w-box-result}, respectively. Additionally, SVHN and Tiny Imagenet results using RN18 are reported in Tables \ref{table:svhn-rn18} and \ref{table:imagenet-rn18}. Experimental outcomes are averaged over three runs with random seeds, and standard deviations are omitted as they are deemed insignificant $(< 0.3)$.

\begin{table}[H]
\caption{Comparing white-box attack robustness (accuracy \%) for RN18 on CIFAR-10.}

\label{table:cifar10-rn18}
\vskip -0.9in
\begin{center}
\begin{small}
\begin{sc}
\begin{tabular}{lcccc}
\hline

Defense & Natural   &PGD-20&CW&AA\\
\hline
 AT & \textbf{84.10}
                     & 52.72 & 51.80 & 47.95\\
 MWPB-AT      & 83.78  & 56.25& 53.02 &49.96\\

 SDWPB-AT      & 83.56  & \textbf{56.69} & \textbf{53.51} &\textbf{50.01}\\
 \hline
 \hline
  MART        & 80.32& 55.15 & 49.35 & 47.63\\                                
 \textbf{MWPB-MART}        & \textbf{82.23}  & 57.10 & 52.57 &\textbf{49.53}\\
 \textbf{SDWPB-MART}        & 80.89& \textbf{57.36} & \textbf{52.66} & 49.45\\ 
 \hline
 \hline
 TRADES        & 82.65   & 52.82& 51.82&  48.96\\
\textbf{MWPB-TRADES} 
        & \textbf{82.89}
            
                           & \textbf{55.53}
                                 &\textbf{53.04} & \textbf{50.73} \\
 \textbf{SDWPB-TRADES}      & 82.69  & 55.36& 52.92 &50.19\\

 \hline                              
\hline

\end{tabular}
\end{sc}
\end{small}
\end{center}
\vskip -0.1in
\end{table}

\begin{table}[H]
\caption{Comparing white-box attack robustness (accuracy \%) for RN18 on SVHN.}
\label{table:svhn-rn18}
\vskip -0.9in
\begin{center}
\begin{small}
\begin{sc}
\begin{tabular}{lcccc}
\hline

Defense & Natural   & PGD-20 & CW  & AA\\
\hline
 AT & \textbf{92.27}
                     & 55.67 & 52.92 & 45.94\\
 \textbf{MWPB-AT}      & 91.45  & \textbf{61.81}& 55.89 &\textbf{49.73}\\
 \textbf{SDWPB-AT}      & 91.06  & 61.59& \textbf{56.40} &48.58\\
 \hline
 \hline
  MART        & \textbf{91.59}& 58.78 & 52.79 & 43.60\\

 \textbf{MWPB-MART}        & 91.51  & \textbf61.87 & 55.11 &48.45 \\
 \textbf{SDWPB-MART}      & 91.33  & \textbf{61.95} & \textbf{55.48} &\textbf{49.45}\\
 \hline
 \hline
 TRADES        & \textbf{90.85}   & 57.27& 53.59&  46.45\\
MWPB-TRADES 
        & 90.35
            
                           & \textbf{60.25}
                                 & 55.03 & 50.11 \\
SDWPB-TRADES      & 90.29  &60.11& \textbf{55.11} &\textbf{50.23}\\

 \hline                              
\hline

\end{tabular}
\end{sc}
\end{small}
\end{center}
\vskip -0.1in
\end{table}

\begin{table}[H]
\caption{Comparing white-box attack robustness (accuracy \%) for RN18 on Tiny Imagenet.}
\label{table:imagenet-rn18}
\vskip -0.3in
\begin{center}
\begin{small}
\begin{sc}
\begin{tabular}{lcccc}
\hline

Defense & Natural   & PGD-20 & CW  & AA \\
\hline
 AT & 48.83
                     & 23.96 & 21.85 & 17.91 \\
 \textbf{MWPB-AT}      & 51.21  & 25.07& 23.11 &19.75\\

\textbf{SDWPB-AT}      & \textbf{51.43}  & \textbf{25.81}& \textbf{23.55} &\textbf{20.45}\\
 \hline
 \hline
  MART        & 46.01& 26.03 & 21.78 & 19.18 \\

 \textbf{MWPB-MART}        & \textbf{47.39}  & 27.15 & 22.89 &20.51\\

 \textbf{SDWPB-MART}        & 46.71  &\textbf{27.55} & \textbf{23.01} &\textbf{20.81}\\
 \hline
 \hline
 TRADES        & 49.11   & 22.82&17.79 &  16.82\\
\textbf{MWPB-TRADES} 
        & \textbf{52.12}
            
                           & 24.60
                                 & 19.85 & 18.15\\
                    
\textbf{SDWPB-TRADES}      & \textbf{53.11}  & \textbf{24.74}& \textbf{20.01} & \textbf{18.34}\\
 \hline                              
\hline

\end{tabular}
\end{sc}
\end{small}
\end{center}
\vskip -0.1in
\end{table}

\subsubsection{Performance on natural examples.}
Training with adversarial examples crafted with larger perturbation budgets ($>8/255$) tends to lower natural accuracy. The warm-up period utilized in our training before introducing \textit{MWPB} and \textit{SDWPB} helped ease the introduction of larger perturbation radii in the later epochs. The warm-up approach, where training starts normally before introducing reweighting or other adaptive methods, is common in related work, e.g., \cite{ding2019mma,liu2021probabilistic,fakorede2023vulnerability}, to mention a few. Here, we train the first 80 epochs with $\epsilon$=4/255 and a step size 1/255 then transition to the adaptive perturbation budgets. Experimental results show that our methods, despite of larger perturbation radii used on some examples, yield better natural accuracy in some cases especially on CIFAR-10 and Tiny Imagenet. We provide more information in Appendix A.1

\subsubsection{Comparison with vanilla baselines}
We compared our proposed method with vanilla baselines  standard AT \cite{madry2017towards}, TRADES \cite{zhang2019theoretically} and MART \cite{wang2019improving}. Experimental results demonstrate that the introduction of \textit{MWPB} and \textit{SDWPB} lead to enhancements in \textit{AT}, \textit{TRADES}, and \textit{MART}. Moreover, our proposed methods exhibit improvements in robust accuracy without compromising natural accuracy. These performance gains are consistent across different datasets and baselines. 
Specifically, when combined with \textit{AT}, \textit{MWPB-AT} showcases notable improvements against adversarial attacks  PGD-20 (+3.06), CW (+2.41), AA($l_{\infty}$) (+2.24) on CIFAR-10 when using WRN34-10. Similarly, on datasets  SVHN and Tiny Imagenet with RN18, \textit{MWPB-AT} outperforms \textit{AT} against PGD-20, CW, and Autoattack. 
The performance of \textit{MWPB} remains consistent when integrated with \textit{TRADES} and \textit{MART}. \textit{MWPB-TRADES} exhibits enhancements in both natural accuracy and robustness against attacks  PGD, CW, and Autoattack. Similarly, \textit{MWPB-MART} shows considerable improvements, especially against Autoattack, on CIFAR-10 when using WRN-34-10 These improvements extend to datasets  SVHN and Tiny Imagenet with RN-18.

Likewise, 
\textit{SDWPB-AT}, \textit{SDWPB-TRADES}, and \textit{SDWPB-MART} improve over \textit{AT, TRADES}, and \textit{MART} respectively. \textit{SDWPB-TRADES} yields a better improvement over TRADES than \textit{MWPB-TRADES} (+0.25) on CIFAR-10 when using WRN-34-10. However, on CIFAR-10, SDWPB appears to marginally reduce the natural accuracy on \textit{AT}, \textit{TRADES}, and \textit{MART}. Experimental results also suggest that SDWPB generally performs better than MWPB on Tiny Imagenet, as may be observed in Table \ref{table:imagenet-rn18}.

Lastly, experimental results also show that our method improves performance on strong black-box attacks  Square and SPSA, summarized in the last two columns of Table \ref{table:wrn-w-box-result}.

\begin{table*}[h!]
\caption{Comparing white-box and black-box attack robustness (accuracy \%) for WRN34-10 on CIFAR-10.}
\label{table:wrn-w-box-result}
\vskip -0.3in
\begin{center}
\begin{small}
\begin{sc}
\begin{tabular}{lcccccccc}
\hline

Defense & Natural   & PGD-20 & CW  & AA ($L_{\infty}$)& AA ($L_2$)& SQUARE & SPSA& \\
&&$\epsilon$ = 8/255&$\epsilon$ = 8/255&$\epsilon$ = 8/255&$\epsilon$ = 128/255&&&\\
\hline
 AT & 86.21
                     & 56.12 & 54.95 & 51.92&58.52& 60.12&61.05  \\
\textbf{MWPB-AT}      & \textbf{86.82}  & 59.18& \textbf{57.36} &\textbf{54.16}&\textbf{61.09}&\textbf{61.15}&63.07& \\
 \textbf{SDWPB-AT}        & 86.09& \textbf{59.36} & 57.04 & 54.08&60.83&60.41&\textbf{63.14}&  \\ 
 \hline
 \hline
  MART        & 84.17& 58.10 & 54.51 & 51.11&57.75&58.74&58.91&  \\

 \textbf{MWPB-MART}        & \textbf{85.70}  & 60.65 & 56.78 &53.80&\textbf{60.41} &\textbf{60.83}&\textbf{62.02}&  \\
 \textbf{SDWPB-MART}        & 85.31& \textbf{60.71} & \textbf{56.82} & \textbf{53.88}&60.09&60.35&61.89&  \\ 
 \hline
 \hline
 TRADES        & 84.70   & 56.30& 54.51&  53.06&58.05& 59.16&61.15&\\
\textbf{MWPB-TRADES} 
        & \textbf{86.09}
            
                           & \textbf{59.10}
                                 & \textbf{57.04}& 54.38&\textbf{60.03} &\textbf{60.77}&62.19 &\\
\textbf{SDWPB-TRADES}        & 85.62& 58.99 & 57.01 & \textbf{54.49}&59.89&60.69&\textbf{62.31}&  \\ 

 \hline                              
\hline

\end{tabular}
\end{sc}
\end{small}
\end{center}
\vskip -0.1in
\end{table*}

\begin{table*}[h!]
\caption{Comparing white-box  and black-box attack robustness (accuracy \%) of various margin-based approaches for WRN34-10 on CIFAR-10, and other prominent baselines}.
\label{table:margin-other}
\vskip -0.3in
\begin{center}
\begin{small}
\begin{sc}
\begin{tabular}{lcccccccc}
\hline

Defense & Natural   & PGD-20 & CW  & AA &  SPSA&\\
\hline
 MMA (\cite{ding2019mma})        & 86.29& 57.12 & 57.59 & 44.52&59.87 & \\ 
 \hline
 WAT (\cite{zeng2021adversarial}) & 85.13
                     & 56.63 & 53.97 & 50.01& 60.75& \\
\hline
 MAIL  (\cite{liu2021probabilistic})      & 86 .81 & 60.49& 51.45 &47.11& 59.25&\\
 
 \hline

GAIRAT (\cite{zhang2020geometry}) 
        & 85.41
            
                           & 60.76
                                 &45.02& 42.29 &52.32 &\\
\hline

\hline
ST-AT (\cite{li2023squeeze}) 
        & 84.91
            
                           & 57.52
                                 &55.11&53.54  &61.34 &\\
\hline

AWP (\cite{wu2020adversarial}) 
        &85.36
            
                           & 58.04
                                 &55.92& 53.92 & 62.57 &\\

 \hline                                
 \textbf{MWPB-AT} (ours)        & 86.85  & 59.18 & 57.36 
 &54.16 &63.07&  \\
 \hline
 \textbf{SDWPB-AT}(ours)      & 86.09 & 59.36& 57.04 &54.08&63.14& \\

\hline
\textbf{MWPB-AWP} (ours + awp)
        & \textbf{87.61}
            
                           & 61.56
                                 &58.54& 56.02 &64.11 &\\
                                                                  
\hline
\textbf{SDWPB-AWP} (ours + awp)
        & 87.59
            
                           & \textbf{61.89}
                                 &\textbf{59.09}& \textbf{56.22} &\textbf{64.22} &\\
 
 \hline

\hline


\end{tabular}
\end{sc}
\end{small}
\end{center}
\vskip -0.1in
\end{table*}

\subsubsection{Comparison with other adaptive radii, margin-based, and recent methods}
We compare our proposed method to  \textit{MMA} \cite{ding2019mma}, which also aims to improve adversarial training by enabling the adaptive selection of the ``correct`` perturbation radii. \textit{MMA} minimizes adversarial losses  at the ``shortest possible perturbation`` for individual examples.  Experimental results displayed in Table \ref{table:margin-other} show that the proposed \textit{MWPB-AT} outperforms \textit{MMA} on natural accuracy (+0.46), PGD-20 (+2.06), AA (+9.64), and SPSA (+3.2), albeit \textit{MMA} slightly performs better  on CW attack (-0.33). Similarly, \textit{SDWPB-AT} performs better than \text{MMA} on PDG-20 (+2.24),  AA (+ 9.56).

It's important to highlight that \textit{MMA} employs a bisection search algorithm to determine optimal perturbation radii for adversarial examples, whereas our approach involves a simpler reweighting of the commonly used perturbation radius. This difference in methodology is worth noting, as the bisection search employed by \textit{MMA} can be computationally more expensive.

Prior works \textit{WAT} \cite{zeng2021adversarial} and \textit{MAIL} \cite{liu2021probabilistic} have incorporated the idea of multi-class margin specified in Eq.(\ref{margin}) to improve adversarial robustness. Specifically, these methods utilize multi-class logit margins to reweight adversarial losses, assigning larger weights to losses corresponding to easily misclassified adversarial examples. Experimental results in Table \ref{table:margin-other} show that our methods perform better than these methods on stronger attacks  CW and Auto attack. Also, these methods have been argued to show signs of gradient obfuscation \cite{fakorede2023vulnerability}, given their low performance on strong black-box attacks. 

The proposed methods also significantly outperform a prominent reweighting approach \textit{GAIRAT}  \cite{zhang2020geometry} on stronger attacks  CW and AA. MWPB-AT and SDWPB-AT improve over GAIRAT on AA by 11.87\%  and 11.79\% , respectively. Similarly, the proposed methods outperform recent work \textit{ST-AT} \cite{li2023squeeze} on natural accuracy and all other attacks evaluated. Combining \textit{AWP} \cite{wu2020adversarial} with MWPB-AT and SDWPB-AT respectively improve over AWP 
on all attacks as  observed in Table~\ref{table:margin-other}. 

\subsubsection{Distribution of perturbation radii}
We show in Fig.~1 in the Appendix the perturbation radii distributions for MWPR-AT/SDWPB-AT, MWPB-TRADES/SDWPB-TRADES, and MWPB-MART/SDWPB-MART for RN-18 over  50,000 training samples of CIFAR-10. The perturbation radii distribution is computed on the best-performing epoch in each case. 

In addition, we added more experimental results showing the robustness of the proposed methods under various perturbation radii and strong adaptive attack FAB \cite{croce2020minimally} in Appendix A.1.


\section{Conclusion}
In this paper, we argue that natural examples, from which adversarial examples are generated, exhibit differing levels of inherent vulnerabilities. As a result, we advocate against the use of uniform perturbations in the inner maximization step of the adversarial training framework. Rather, we propose instance-specific weighting functions for determining the perturbation budgets when crafting  adversarial examples for adversarial training. The weighting function assesses the vulnerability of each natural example and utilizes this information for determining perturbation radii when generating adversarial examples. Experimental results show that our proposed approach consistently enhances the performance of popular adversarial training methods across various datasets and under different attacks. 

\bibliography{tmlr}
\bibliographystyle{tmlr}

\appendix
\section{Appendix}
\subsection{Further Experiments}

\subsubsection{Effect of warm-up period on natural accuracy.}
In an ideal scenario, larger perturbation budgets typically lead to a reduction in natural accuracy. However, we addressed this by initiating training with a smaller perturbation budget ($\epsilon = 4/255$) and a step size of $1/255$ until the 80th epoch. Our experimental results illustrate how this training approach effectively mitigates the impact on natural accuracy in the proposed MWPB. Specifically, the experiments presented below were conducted on WRN-34-10 on CIFAR-10.

\begin{table}[H]
\caption{Studies on \textit{AT}, \textit{MWPB-AT} and \textit{SDWPB-AT} showing the influence of the warming up training before applying reweighting. } 
\label{table:ablation-mwpb-at}
\setlength{\tabcolsep}{3pt}
\vskip 0.05in
\begin{center}
\begin{small}
\begin{sc}
\begin{tabular}{lcccc}
\hline
\hline
Method& Natural & PGD-20& AA \\
\hline

      AT& 	86.21 & 56.21&51.92\\
  
   MWPB-AT (no warm-up)&84.39 & 58.71 & 53.87 \\
   SDWPB-AT (no warm-up)&83.09 & 59.19 & 53.84 \\

   MWPB-AT (with warm-up) &86.85& 59.18 & 54.16  \\
   SDWPB-AT (with warm-up) &86.02& 59.36 & 54.08 \\

 \hline
 \hline
\end{tabular}
\end{sc}
\end{small}
\end{center}
\vskip -0.1in
\end{table}

\vskip 0.2in

\subsubsection{Robustness against PGD-20 under various perturbation sizes and FAB attack }
Given that our approach rely on training models with adversarial samples crafted using adaptive perturbation budget, we study the robustness of the proposed methods under various $l_{\infty}$-bounded perturbation sizes. Furthermore, we tested against a strong attack, FAB \cite{croce2020minimally}, which adaptively tracks the decision boundary with the aim of changing the class of a given input. Our experimental results displayed in Table \ref{varying-pert}, show that combining \textit{MWPB} and \textit{SDWPB} with AT, TRADES and MART consistently yield improved robustness to PGD attacks crafted under varied adversarial perturbations. In addition, our methods improve the robustness of \textit{AT}, \textit{TRADES} and \textit{MART} respectively to \textit{FAB} attack.

\begin{table*}[!h]
\caption{Comparing white-box attack robustness (accuracy \%) for WRN-34-10 on CIFAR-10 under various $l_{\infty}$-bounded  perturbation sizes  for PGD-20 and FAB attack.}
\label{varying-pert}
\vskip -0.3in
\begin{center}
\begin{small}
\begin{sc}
\begin{tabular}{lcccccccc}
\hline

Defense & PGD-20  & PGD-20 & PGD-20   &  FAB\\
& $\epsilon$ = 6/255 &$\epsilon$ = 10/255& $\epsilon$ = 12/255&$\epsilon$ = 128/255&\\
\hline
 AT & 64.49
                     & 49.58 & 45.29 &52.02\\
 MWPB-AT      & \textbf{67.12}  & \textbf{52.08}& \textbf{48.81}& \textbf{54.52}& \\
 \hline
 \hline
  MART        & 65.41& 52.31 & 48.84 &52.18\\

 MWPB-MART        & \textbf{67.19}  & \textbf{55.63} & \textbf{51.96} &\textbf{54.25}\\
 \hline
TRADES        & 64.19   & 51.92& 48.23& 53.59\\
MWPB-TRADES 
        & \textbf{65.82}
            
                           & \textbf{53.58}
                                 &\textbf{50.11}  &\textbf{55.17}\\

 \hline                              

\end{tabular}
\end{sc}
\end{small}
\end{center}
\vskip 0.5in
\end{table*}

\subsection{Distribution of perturbation radii}
We show in Fig. \ref{histogram}, the perturbation radii distribution for MWPB-AT, MWPB-TRADES, and MWPB-MART for RN-18 over the 50,000 training samples of CIFAR-10. The perturbation radii distribution is computed on the best-performing epoch in each case. Experimental results show that MWPB-AT has utilized the minimum perturbation radii of 0.018 and the maximum perturbation radii of 0.0552. MWPB-MART and MWPB-TRADES have minimum perturbation radii of 0.0200 and 0.02253, respectively. The maximum perturbation radii of MWPB-MART and MWPB-TRADES are respectively 0.0554 and 0.0471. 

Similarly, we show in Fig. (\ref{sdwpb-histogram}) the perturbation radii distribution for SDWPB-AT, SDWPB-TRADES, and SDWPB-MART, respectively. Our experimental results show that the adversarial example with the minimum perturbation has radii of 0.0317 and maximum perturbation radii of 0.057 for SDWPB-AT. The adversarial examples used for training SDWPB-TRADES and SDWPB-MART have minimum perturbations of 0.0313  and 0.0312. The maximum perturbation radii of SDWPB-MART and SDWPB-TRADES are respectively 0.0485 and 0.0498.

\subsection{Impact of $\alpha$ hyperparameter.}

In this analysis, we investigate the influence of the hyperparameter $\alpha$ in the reweighting functions introduced in Equations (4) and (5). $\alpha$ plays a crucial role in controlling the strength of the reweighting functions. Our experiments across various combinations of \textbf{MWPB} and \textit{SDWPB} with \textit{AT}, \textit{TRADES}, and \textit{MART} reveal a trade-off. When $\alpha$ is set to a low value, natural accuracies tend to be relatively higher, but robust accuracies are relatively lower. Conversely, as the value of $\alpha$ increases, robust accuracy on PGD-20 and Auto-attack improves, but at the expense of natural accuracy. Our selection of the optimal $\alpha$ values is guided by finding a balance that ensures a reasonable trade-off between natural and robust accuracy, particularly in the case of Autoattack robust accuracy. We presents the experimental results in Tables (\ref{table:ablation-mwpb-at}) - (\ref{table:ablation-sdwpb-mart}).

\begin{table}[!hbt]
\caption{Studies on \textit{MWPB-AT} showing the impact of the $\alpha$  hyperparameter. } 
\label{table:ablation-mwpb-at}
\setlength{\tabcolsep}{3pt}
\vskip 0.05in
\begin{center}
\begin{small}
\begin{sc}
\begin{tabular}{lcccc}
\hline
\hline
$\alpha$& Natural & PGD-20& AA &\\
\hline

      0.10& 88.93 & 56.25&51.67&\\
  
   0.20&88.56 & 56.85 & 51.88& \\

   0.30 &88.29& 57.80 & 52.58&  \\
   0.40 &87.72 & 58.11 & 52.85& \\
    0.50  &87.40 & 58.78 & 53.64&  \\
     \textbf{0.58}  &86.82 & 59.18 & 54.16& \\
      0.70  &85.55 & 59.12 & 54.14&  \\
       0.85  &85.11 & 59.39 & 54.02& \\
        1.2  &83.75 & 59.69 & 53.50& \\

 \hline
 \hline
\end{tabular}
\end{sc}
\end{small}
\end{center}
\vskip -0.1in
\end{table}

\begin{table}[H]
\caption{Studies on \textit{MWPB-TRADES} showing the impact of the $\alpha$  hyperparameter. } 
\label{table:ablation-mwpb-trades}
\setlength{\tabcolsep}{3pt}
\vskip 0.05in
\begin{center}
\begin{small}
\begin{sc}
\begin{tabular}{lcccc}
\hline
\hline
$\alpha$& Natural & PGD-20& AA &\\
\hline

      0.1& 86.82 & 58.51&53.76&\\
  
   0.2& 86.21& 58.63 &54.01 & \\

   0.3 &86.11& 58.82 & 54.27&  \\
   \textbf{0.42} &86.09 & 59.10 & 54.38& \\
    0.5  &85.63 & 58.91 & 54.31&  \\
    0.6  &85.41 & 58.87 & 54.22& \\
    0.85  &84.43 & 58.63 & 54.05& \\

 \hline
 \hline
\end{tabular}
\end{sc}
\end{small}
\end{center}
\vskip -0.1in
\end{table}

\begin{table}[H]
\caption{Studies on \textit{MWPB-MART} showing the impact of the $\alpha$  hyperparameter. } 
\label{table:ablation-mwpb-mart}
\setlength{\tabcolsep}{3pt}
\vskip 0.05in
\begin{center}
\begin{small}
\begin{sc}
\begin{tabular}{lcccc}
\hline
\hline
$\alpha$& Natural & PGD-20& AA &\\
\hline

      0.1& 88.21 & 59.17&52.16&\\
  
   0.2&87.79 & 59.69& 52.63& \\

   0.3 &87.18& 60.04 & 52.81&  \\
   0.4 &86.76 &  60.19& 53.32& \\
   0.5&86.13&60.48&53.71\\
   \textbf{0.55}  &85.70 & 60.65 & 53.80&  \\
     0.65 &85.11 & 60.82 & 53.68& \\

 \hline
 \hline
\end{tabular}
\end{sc}
\end{small}
\end{center}
\vskip -0.1in
\end{table}

\begin{table}[H]
\caption{Studies on \textit{SDWPB-AT} showing the impact of the $\alpha$  hyperparameter.}  
\label{table:ablation-sdwpb-at}
\setlength{\tabcolsep}{3pt}
\vskip 0.05in
\begin{center}
\begin{small}
\begin{sc}
\begin{tabular}{lcccc}
\hline
\hline
$\alpha$& Natural & PGD-20& AA &\\
\hline

      0.10& 88.85 & 56.32&51.51&\\
  
   0.20&88.76 & 56.91 & 51.75& \\

   0.30 &88.16& 57.93 & 52.29&  \\
   0.40 &87.23 & 58.52 & 52.83& \\
    0.50  &86.58 & 58.92 & 53.51&  \\
     \textbf{0.62}  &86.02 & 59.36 & 54.08& \\
      0.70  &85.39 & 59.35 & 53.94&  \\
       0.85  &84.65 & 59.71 & 53.90& \\
        1.2  &82.13 & 59.88 & 53.19& \\

 \hline
 \hline
\end{tabular}
\end{sc}
\end{small}
\end{center}
\vskip -0.1in
\end{table}

\begin{table}[H]
\caption{Studies on \textit{SDWPB-TRADES} showing the impact of the $\alpha$  hyperparameter. } 
\label{table:ablation-sdwpb-trades}
\setlength{\tabcolsep}{3pt}
\vskip 0.05in
\begin{center}
\begin{small}
\begin{sc}
\begin{tabular}{lcccc}
\hline
\hline
$\alpha$& Natural & PGD-20& AA &\\
\hline

      0.1& 86.64 & 58.57&53.71&\\
  
   0.2&86.10 & 58.64& 53.89& \\

   0.3 &85.91& 58.71 & 54.09&  \\
   0.4 &85.73 & 58.81 & 54.22& \\
   \textbf{0.52}   &85.62 & 58.99 & 54.49&  \\
     0.6  &84.91 & 58.75 & 54.31& \\
     0.85&84.67&58.49&54.13&\\

 \hline
 \hline
\end{tabular}
\end{sc}
\end{small}
\end{center}
\vskip -0.1in
\end{table}

\begin{table}[H]
\caption{Studies on \textit{SDWPB-MART} showing the impact of the $\alpha$  hyperparameter. } 
\label{table:ablation-sdwpb-mart}
\setlength{\tabcolsep}{3pt}
\vskip 0.05in
\begin{center}
\begin{small}
\begin{sc}
\begin{tabular}{lcccc}
\hline
\hline
$\alpha$& Natural & PGD-20& AA &\\
\hline

      0.1& 87.67 & 59.33&52.21&\\
  
   0.2&86.95 & 59.97& 52.72& \\

   0.3 &86.34& 60.11 & 52.91&  \\
   0.4 &86.17 &  60.49& 53.26& \\
   \textbf{0.52}   &85.31 & 60.77 & 53.88&  \\
     0.6  &84.76 & 60.72 & 53.74& \\

 \hline
 \hline
\end{tabular}
\end{sc}
\end{small}
\end{center}
\vskip -0.1in
\end{table}

\begin{figure*}[!h]
\centering    
\includegraphics[width=1.04\textwidth]{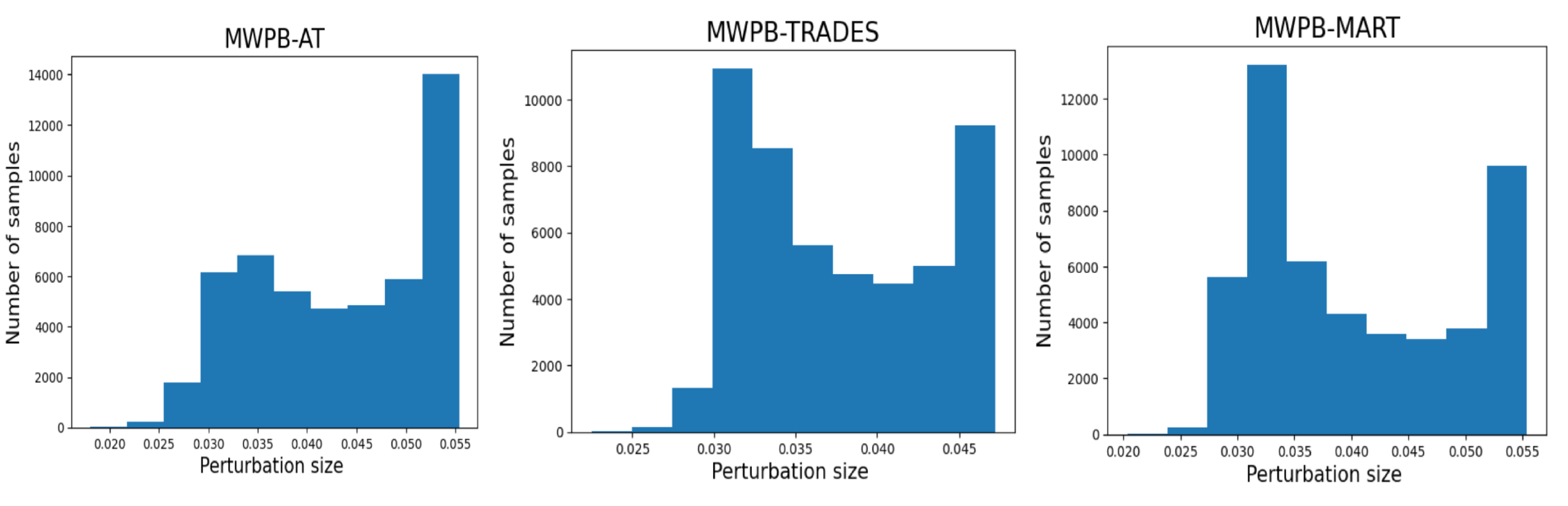}
\caption{Plots showing the distribution of perturbation radii for MWPB-AT, MWPB-TRADES and MWPB-MART respectively.\label{histogram} }
\vskip 0.1in
\end{figure*}

\begin{figure*}[!h]
\centering    
\includegraphics[width=1.04\textwidth]{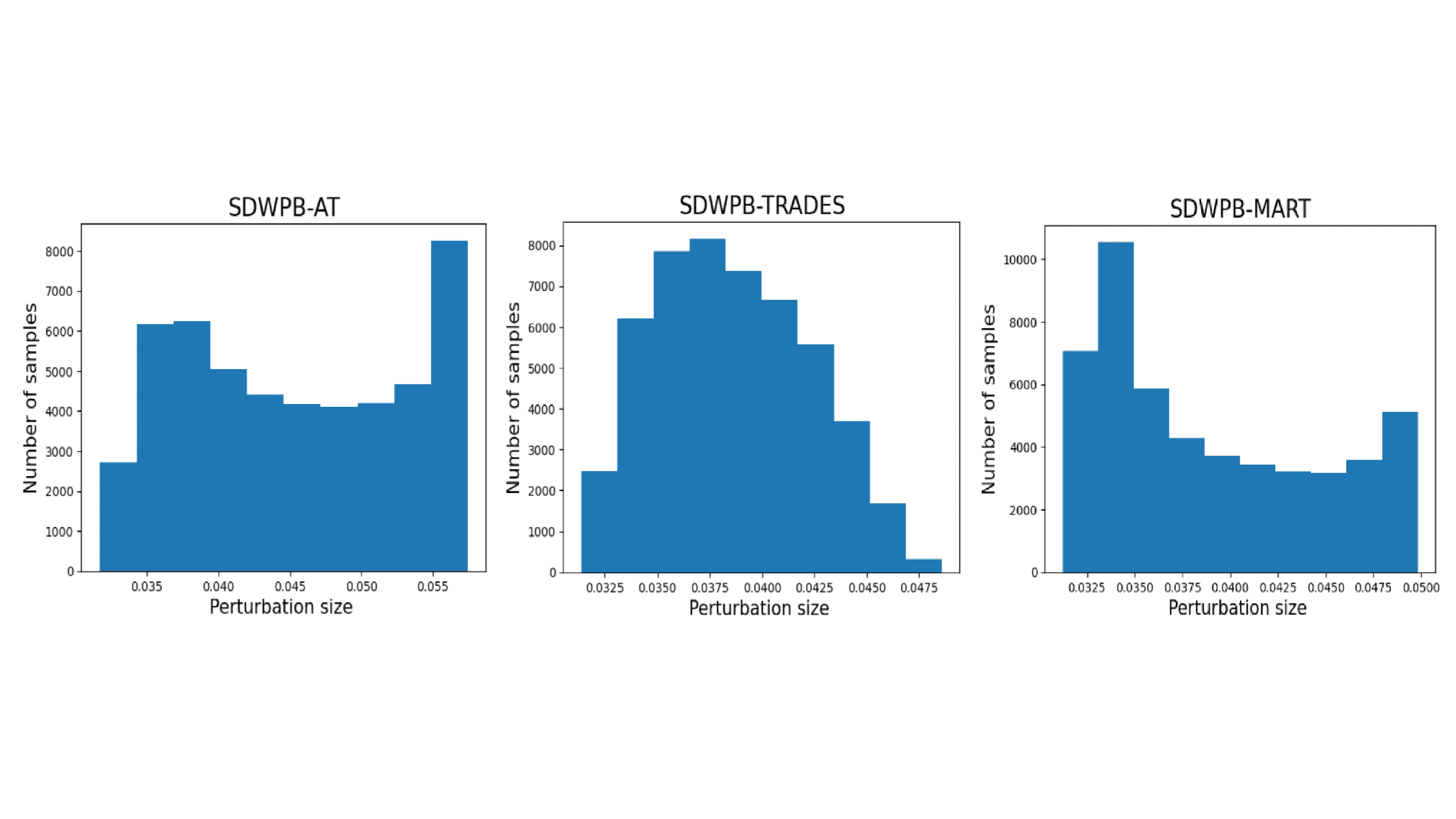}
\caption{Plots showing the distribution of perturbation radii for SDWPB-AT, SDWPB-TRADES and SDWPB-MART respectively.\label{sdwpb-histogram} }
\vskip 0.1in
\end{figure*}

\newpage
\section{Appendix}
\textit{\textbf{Proof of Theorem 1}}

The PGD attack utilized for inner maximization is a first-order adversary \cite{madry2017towards,simon2019first}.

\begin{lemma}[\cite{pang2020boosting}]
Given a loss function $\mathcal{L}$ and under the first-order Taylor expansion, the solution to the inner maximization:
\[ \max_{\textbf{x}' \in B_{\epsilon}(\textbf{x})} L(f_{\theta}(\textbf{x}'), y)\]
is $x^* = x + \epsilon \mathbb{U}_p({\nabla}_x \mathcal{L}(x))$. Furthermore, $\mathbf{L}(f(x^*), y) = \mathcal{L}(f(x), y) + \epsilon {\| {\nabla}_x \mathcal{L}(x) \|}_q$, where ${\|.\|}_q$ is the dual norm of ${\|.\|}_p$
\end{lemma}

\begin{lemma}[\cite{katharopoulos2017biased}]
Let $\mathcal{L}$ be either a negative log-likelihood or square error loss function. Then, $\mathcal{L}(f_{\theta}(x_1), y_1) > \mathcal{L}(f_{\theta}(x_2), y_2)$ $\iff$ $\| {\nabla}_x \mathcal{L}(f_{\theta}(x_1), y_1)\| >  \| {\nabla}_x \mathcal{L}(f_{\theta}(x_2), y_2)\|$.
    
\end{lemma}

\begin{proof}
    We anchor the proof of Theorem 1 on Lemma 1 and Lemma 2.

Given two input-label pairs $(x_1, y_1)$ and $(x_2, y_2)$, such that $\mathcal{L}(f_{\theta}(x_1), y) > \mathcal{L}(f_{\theta}(x_2), y)$. Then according to Lemma 1, the inner maximization of $\mathcal{L}(f_{\theta}(x_1), y_1)$ and $\mathcal{L}(f_{\theta}(x_2), y_2)$ under the same perturbation bound $\epsilon$ are given as:
\[\mathcal{L}(f_{\theta}(x^*_1), y_1) = \mathcal{L}(f_{\theta}(x_1), y_1) + \epsilon {\| {\nabla}_{x_1} \mathcal{L}f_{\theta}(x_1, y_1) \|}_q \]
\[\mathcal{L}(f_{\theta}(x^*_2), y_2) = \mathcal{L}(f_{\theta}(x_2), y_2) + \epsilon {\| {\nabla}_{x_2} \mathcal{L}f_{\theta}(x_2, y_2) \|}_q \]

It can seen that the increment in loss resulting from the inner maximization depend on $ \epsilon {\| {\nabla}_{x_1} \mathcal{L}(f_{\theta}(x_1), y_2) \|}_q$ and $ \epsilon {\| {\nabla}_{x_2} \mathcal{L}(f_{\theta}(x_2), y_2) \|}_q$. Since $\epsilon$ is constant then the loss increment depend on norm of the gradient. Also according to Lemma 2, ${\| {\nabla}_x L(f_{\theta}(x_1), y_1)\|}_q >  {\| {\nabla}_x L(f_{\theta}(x_2), y_2)\|}_q$. Therefore under a fixed perturbation radius $\epsilon$,
\begin{enumerate}
    \item  $\mathcal{L}(f_{\theta}(x^*_1), y_1)$ $>$ $\mathcal{L}(f_{\theta}(x^*_2), y_2)$
    \item  $\mathcal{L}(f_{\theta}(x^*_1), y_1) - \mathcal{L}(f_{\theta}(x_1), y_1) $  $>$ $\mathcal{L}(f_{\theta}(x^*_2), y_2) - \mathcal{L}(f_{\theta}(x_2), y_2)$
\end{enumerate}

\end{proof}

\textit{\textbf{Proof of Theorem 2}}

\begin{proof}
   
    Consider an input-label pair $(x, y)$ and a loss function $\mathcal{L}$ with the inner maximization \[ \max_{\textbf{x}' \in B_{\epsilon}(\textbf{x})} L(f_{\theta}(\textbf{x}'), y)\] and the solution \[L(f(x^*), y) = L(f(x), y) + \epsilon {\| {\nabla}_x L(f_{\theta}(\textbf{x}), y) \|}_q \]. Then, given a perturbation radius ${\epsilon}_g$ $>$ $\epsilon$, we have: \[L(f(x^*), y) = L(f(x), y) + {\epsilon}_g{\| {\nabla}_x L(f_{\theta}(\textbf{x}), y) \|}_q\]. Since $L(f(x), y)$ and ${\| {\nabla}_x L(f_{\theta}(\textbf{x}'), y) \|}_q$ remain fixed, and ${\epsilon}_g$ $>$ $\epsilon$, then $L(f(x^*), y)$ computed within ${\epsilon}_g$ is greater than $L(f(x^*), y)$ within $\epsilon$.
\end{proof}

\section{Appendix}
Here we present the proposed algorithm for MWPB-AT
\begin{algorithm}[H]
\caption{MWPB-AT Algorithm.}\label{alg:mwpb}
\hspace*{\algorithmicindent} \textbf{Input:} a neural network model with the parameters $\theta$, step sizes $\kappa_i$ and $\kappa$, and a training dataset $\mathcal{D}$ \hspace*{\algorithmicindent} of size n.\\
\hspace*{\algorithmicindent} \textbf{Output:} a robust model with parameters $\theta^*$
\begin{algorithmic}[1] 
\State  \textbf{set} $\epsilon = 8/255$
\For{$epoch = 1$ to num\_epochs}
\For{$batch = 1$ to num\_batchs}
\State sample a mini-batch $\{(x_i, y_i)\}_{i=1}^{M}$ from $\mathcal{D}$;\Comment{mini-batch of size $M$.}

\For{$i = 1$ to M} 
\State $d_{m}(\textbf{x}_i, y_i; \theta)= f_{\theta}(\textbf{x}_i)_{y_i}  - \max_{ k, k \neq {y_i}} f_{\theta}(\textbf{x}_i)_k$
\State  ${\epsilon}_{i} = exp(\alpha \cdot d_{m}(\textbf{x}_i, y_i; \theta)) * \epsilon$
 \State ${\kappa}_i  = {\epsilon}_i / 4$
\State $\mathbf{x}_i^{'}$ $\leftarrow$ $\mathbf{x}_i$ + 0.001 $\cdot$  $\mathcal{N}(0, 1)$;   \Comment{$ \mathcal{N}(0, I)$ is a Gaussian distribution with zero mean and identity variance.}

\For{$k = 1$ to $K$}

\If{$epoch\leq 80$} 
    \State $\mathbf{x}_{i}' \leftarrow \prod_{ B_{{\epsilon} / 2}(\mathbf{x}_i)}(x_{i} + {\kappa} / 2 \cdot sign (\nabla_{\mathbf{x}_{i}'}  L(f_{\theta}(\mathbf{x}_{i}'), y_i) )$;  \Comment{$\prod$ is a projection operator.}
\Else
    \State $\mathbf{x}_{i}' \leftarrow \prod_{ B_{{\epsilon}_i} (\mathbf{x}_i)}(x_{i} + {\kappa}_i \cdot sign (\nabla_{\mathbf{x}_{i}'} L(f_{\theta}(\mathbf{x}_{i}'), y_i) )$
    
\EndIf

\EndFor

\EndFor

\State $\theta$ $\leftarrow$  $\theta - \kappa  \nabla_\theta\sum_{i=1}^M  L(f_{\theta}(\mathbf{x}_{i}'), y_i)$ 
\EndFor
\EndFor

\end{algorithmic}
\end{algorithm}

\begin{algorithm}[H]
\caption{SDWPB-AT Algorithm.}\label{alg:sdwpb}
\hspace*{\algorithmicindent} \textbf{Input:} a neural network model with the parameters $\theta$, step sizes $\kappa_i$ and $\kappa$,a training dataset $\mathcal{D}$ \hspace*{\algorithmicindent} of size n and $|C|$ is the number of classes.\\
\hspace*{\algorithmicindent} \textbf{Output:} a robust model with parameters $\theta^*$
\begin{algorithmic}[1] 
\State  \textbf{set} $\epsilon = 8/255$
\For{$epoch = 1$ to num\_epochs}
\For{$batch = 1$ to num\_batchs}
\State \small sample a mini-batch $\{(x_i, y_i)\}_{i=1}^{M}$ from $\mathcal{D}$;\Comment{mini-batch of size $M$.}

\For{$i = 1$ to M} 
\small \State $d_{std}(\textbf{x}_i, y_i; \theta) = \{ \sum_{k=1}^C\frac{(f_{\theta}(\textbf{x}_i)_k - f_{\theta}(\textbf{x}_i)_{y_i})^2)}{|C|}\}^{0.5}$
\State  ${\epsilon}_{i} = exp(\alpha \cdot d_{std}(\textbf{x}_i, y_i; \theta)) * \epsilon$
 \State ${\kappa}_i  = {\epsilon}_i / 4$
\State $\mathbf{x}_i^{'}$ $\leftarrow$ $\mathbf{x}_i$ + 0.001 $\cdot$  $\mathcal{N}(0, 1)$;   \Comment{\footnotesize $ \mathcal{N}(0, I)$  is a Gaussian distribution with zero mean and identity variance.}

\For{$k = 1$ to $K$}

\If{$epoch\leq 80$} 
    \State $\mathbf{x}_{i}' \leftarrow \prod_{ B_{{\epsilon} / 2}(\mathbf{x}_i)}(x_{i} + {\kappa} / 2 \cdot sign (\nabla_{\mathbf{x}_{i}'}  L(f_{\theta}(\mathbf{x}_{i}'), y_i) )$;  \Comment{$\prod$ is a projection operator.}
\Else
    \State $\mathbf{x}_{i}' \leftarrow \prod_{ B_{{\epsilon}_i} (\mathbf{x}_i)}(x_{i} + {\kappa}_i \cdot sign (\nabla_{\mathbf{x}_{i}'} L(f_{\theta}(\mathbf{x}_{i}'), y_i) )$
    
\EndIf

\EndFor

\EndFor

\State $\theta$ $\leftarrow$  $\theta - \kappa  \nabla_\theta\sum_{i=1}^M  L(f_{\theta}(\mathbf{x}_{i}'), y_i)$ 
\EndFor
\EndFor

\end{algorithmic}
\end{algorithm}

\end{document}